\documentclass[runningheads]{llncs}

\usepackage{url}
\usepackage{graphicx}

\usepackage{amssymb}
\usepackage{amsmath}
\usepackage{booktabs}
\usepackage{rotating}
\usepackage{multirow}
\usepackage{tabularx}
\newcolumntype{Y}{>{\raggedleft\arraybackslash}X}

\usepackage{pgf}
\usepackage{tikz}
\usetikzlibrary{arrows,automata}

\makeatletter
\newcommand*{\centerfloat}{%
  \parindent \z@
  \leftskip \z@ \@plus 1fil \@minus \textwidth
  \rightskip\leftskip
  \parfillskip \z@skip}
\makeatother

\newcommand{\cons}[1]{\texttt{#1}}

\newcommand{\Asn}{\ensuremath{\mathit{Asn}}}
\newcommand{\Con}{\ensuremath{\mathit{Con}}}
\newcommand{\Dom}{\ensuremath{\mathit{Dom}}}
\newcommand{\dom}{\ensuremath{\mathit{dom}}}
\newcommand{\sol}{\ensuremath{\mathit{sol}}}
\newcommand{\solve}{\ensuremath{\mathit{solve}}}
\newcommand{\Val}{\ensuremath{\mathit{Val}}}
\newcommand{\Var}{\ensuremath{\mathit{Var}}}

\begin{document}

\title{Half-checking propagators}
\author{Mikael Zayenz Lagerkvist\inst{1}\orcidID{0000-0003-2451-4834}
  \and \\ Magnus Rattfeldt\inst{2}\orcidID{0000-0001-5036-3107}}
\authorrunning{M. Z. Lagerkvist and M. Rattfeldt}
\institute{
  \email{research@zayenz.se}\\
  \url{https://zayenz.se}
  \and
  \email{research@rattfeldt.se}
}

\maketitle

\begin{abstract}
  Propagators are central to the success of constraint programming,
  that is contracting functions removing values proven not to be in
  any solution of a given constraint.
  The literature contains numerous propagation algorithms, for many
  different constraints, and common to all these propagation
  algorithms is the notion of correctness: only values that appear in
  \emph{no solution} to the respective constraint may be removed.

  In this paper \emph{half-checking propagators} are introduced, for
  which the only requirements are that identified solutions (by the
  propagators) are actual solutions (to the corresponding
  constraints), and that the propagators are contracting.
  In particular, a half-checking propagator may \emph{remove
    solutions} resulting in an incomplete solving process, but with
  the upside that (good) solutions may be found faster.
  Overall completeness can be obtained by running half-checking
  propagators as one component in a portfolio solving process.
  Half-checking propagators opens up a wider variety of techniques to
  be used when designing propagation algorithms, compared to what is
  currently available.

  A formal model for half-checking propagators is
  introduced, together with a detailed description of how to support
  such propagators in a constraint programming system.
  Three general directions for creating half-checking propagation
  algorithms are introduced, and used for designing new half-checking
  propagators for the \cons{cost-circuit} constraint as examples.
  The new propagators are implemented in the Gecode system.
\end{abstract}

\section{Introduction}
\label{sec:introduction}

Constraint programming has been successful in a wide variety of
settings, and central to the success of constraint programming is
the multitude of smart and efficient propagation algorithms
devised. Propagation is all about removing values that are not in any
solution to a constraint, and it is what separates constraint
programming from generate-and-test. In constraint programming, we are
justifiably proud of being able to effectively combine algorithms from
many different fields implemented as propagators, so that a model
effortlessly and without fear of adverse interactions can use
intelligent scheduling algorithms for \cons{disjunctive} and
\cons{cumulative} such as not-first/not-last and energetic reasoning,
dynamic programming algorithms for \cons{regular}, \cons{bin-packing},
and \cons{knapsack}, maximum flow reasoning for
\cons{global-cardinality}, arithmetic reasoning for arithmetic
constraints, and Boolean reasoning, among many more.

Unfortunately, designing good propagation algorithms is hard. It is
hard not only since the specific problems they model are hard, but
they are hard for a more fundamental reason. Propagators are required
to be \emph{correct}; they must never remove a value from a variable
that may still be a solution to the constraint. This means that
propagation is not actually concerned with finding a solution but about
proving that no solution exists for a certain variable-value pair,
which is a subjectively harder problem. The requirement for
correctness also means that there is an upper limit on the amount of
propagation that can be done, and this limit (\emph{domain
  consistency}~\cite{mackworth77}) is
often the ultimate goal when designing a new
propagator. Unfortunately, even if a propagator is domain
consistent it does not mean that it performs a high amount of
propagation: perhaps all values can still be part of some solution for
the constraint.

In this paper we propose a new type of propagators, that we call
\emph{half-checking propagators}. By relaxing the requirements of
propagators to a bare minimum for ensuring soundness (found solutions
must be constraint solutions), we open up for a wider variety of
techniques that may be used when designing propagation algorithms. On
the downside, such propagators are no longer correct, which means that
the overarching solving process is no longer complete. On the upside,
however, such propagators can deploy new and stronger reasoning
(possibly even stronger than domain consistency), with the hope that
the search is then guided towards promising parts of the search space.

In many industrial applications, finding a provably optimal solution
is not as interesting as finding solutions that improves the best
known result.  Local search is a typical example of an incomplete
method used for finding \emph{better} solutions, as are heuristics
and approximation algorithms.  In constraint programming, the perhaps
most well known and successful incomplete technique is
Large Neighborhood Search~\cite{lns}. In contrast to these incomplete
methods we embrace the incompleteness earlier by lifting it into the
propagators, the heart of a constraint programming solver. Similar to
all incomplete strategies, completeness can be regained by combining
one or more incomplete solution methods with one or more complete
solution methods in a portfolio solver.

\paragraph{Contributions.} This paper introduces the novel concept of
\emph{half-checking propagators}, including a formal model, a full
exploration on how to integrate into a realistic system, and how to
use in a portfolio solver. Three general techniques for designing
half-checking propagators are defined. For all three, an example
propagator using the technique is developed for the
\cons{cost-circuit} constraint. An implementation in an
industrial-strength constraint programming system (Gecode) has been
made verifying the approach. 

\paragraph{Plan of paper.} In the next section an overview of
constraint programming is given. Sect.~\ref{sec:hf-propagators} gives
a formal model for half-checking propagators, and the next Section
describes the practical aspects of integrating half-checking
propagators in a realistic system. Sect~\ref{sec:tsp} gives a
background on the TSP problem used in the examples, and
Sections~\ref{sec:hf-examples:dominating}
to~\ref{sec:hf-examples:solutions} introduce three techniques for
defining half-checking propagators, with examples using TSP. Some
experimental evaluation is reported in
Section~\ref{sec:evaluation}. Finally, related work is presented and
then our results are summarized in the conclusions.

\section{Constraint programming}
\label{sec:cp}

In order to be clear about the specifics, a formal model of constraint
programming is needed, as is knowing the standard requirements on
propagators.

Let $\mathcal{P}(s)$ be the \emph{power-set} of $s$,
that is the set of all subsets of $s$. The set of all functions from
the set $A$ to the set $B$ is denoted $A\to B$. Let $\lambda x.E$ be the
function from the argument $x$ to the expression $E$. 

\subsection{Constraint satisfaction problems}
\label{sec:cp::formal}

A constraint satisfaction problem is defined over a finite set of
\emph{variables} $\Var=\{x_1,\ldots,x_n\}$ and a finite set of
\emph{values} $\Val$. An \emph{assignment} $a\in \Asn$
maps each variable in $\Var$ to a value in $\Val$, $\Asn=\Var\to \Val$. For a set of
variables $x\subseteq \Var$, $\Asn_x$ is the assignments where the
arguments are restricted to $x$, and $a_x$ is similarly an assignment restricted to
$x$. A \emph{constraint} $c\in \Con$ over variables
$var(c)\subseteq \Var$ is defined as the set of assignments that are
solutions to that constraint:
$\Con = \cup_{x\subseteq\Var}\mathcal{P}(\Asn_x)$.
When necessary and without loss of generality, any constraint is
extended to all variables $\Var$ by allowing all combinations of
values for the added variables, for all solutions.

A \emph{domain} $d\in \Dom$ maps each variable to a subset of the
values, $\Dom=\Var\to \mathcal{P}(\Val)$. For simplicity, all domains
where at least one variable is mapped to the empty set are equated and
represented by the fully empty domain ($\bot=\lambda x. \{\}$). A
domain $d$ induces a set of assignments
($asn(d)=\{a\ |\ \forall x.\, a(x) \in d(x)\}$), and can thus be
considered as a constraint. Domains are ordered and behave similar to
sets by lifting the operations and relations point-wise over the
variables, and is extended to include constraints and assignments
using the induced constraint for the domain.

The domain of a constraint is defined as
$\dom(c)=\lambda x. \{v\ |\ \exists a\in c.\, a(x) = v\}$. Note that
the domain of a constraint in turn induces a much weaker constraint
than the original. For example, the equality constraint $eq$ for two
variables contains just $|\!\Val|$ assignments, while $asn(\dom(eq))$
contains all $|\!\Val|^2$ assignments.

A constraint satisfaction problem (\emph{CSP}) is a tuple
$\langle d, C \rangle$ of a domain $d$ and a set of constraints
$C$. An assignment $a$ is a \emph{solution} to a CSP iff $a\in d$ and
$\forall c\in C.\, a\in c$. The set of all solutions to a CSP $csp$ is
given by the function $\sol(csp)$. A function
$\solve\in CSP\to \mathcal{P}(\Asn)$ finds solutions for a CSP. Such
a function is \emph{sound} iff $\solve(csp)\subseteq sol(csp)$ (all
solutions found are actually solution). It is \emph{complete} iff
$\solve(csp)=\sol(csp)$ (solving finds all solutions).

\subsection{Propagators and models}
\label{sec:cp:propagators}

A \emph{propagator} $p$ for a constraint $c$ is a function\footnote{As
  remarked in~\cite{SchulteTack:CP:2009}, propagators do not need to be
  functions, and can be arbitrary relations in $\Dom\times \Dom$,
  e.g., as a model for randomized propagation. For ease of explanation
  and notation, we use functions as the terminology, and leave
  generalization unstated.} from domains to domains ($p\in\Dom\to\Dom$), with the
following properties.

\begin{description}
\item[Contracting] For all propagators $p$ and domains $d$,
$p(d)\subseteq d$ must hold. 
\item[Local] For all $d\in \Dom$, if $x\not\in var(c)$, then $p(d)(x) = d(x)$.
\item[Checking] For all $a\in \Asn$, $p(dom(\{a\}))=dom(\{a\})$
  iff $a\in c$.
\item[Weakly monotonic] For all $d\in \Dom$ and assignments $a\in d$,
  $p(\{a\})\subseteq p(d)$.
\end{description}

\emph{Contracting} means that a propagator only removes values from
domains, never adds values. \emph{Local} means that a propagator only
removes values from the variables involved in the
constraint. \emph{Checking} means that a propagator recognizes all
solutions to a constraint since no values are removed for those
assignments. \emph{Weakly monotonic} means that if an assignment is a
fix-point of a propagator (and thus a solution to the constraint),
then the propagator does not remove that assignment from a domain it
is in.

Correctness is a crucial property for propagators. it means that no
solution is removed by running a propagator. Any propagator that is
weakly monotonic and checking is correct for its
constraint~\cite{SchulteTack:CP:2009}.
\begin{definition}[Correct] \label{def:correct} A propagator $p$ is
  correct for constraint $c$, iff
  $$\forall a\in c.\, \forall d\in \Dom.\, a\in asn(d)\implies a\in
  asn(p(d))$$
\end{definition}

Let the constraint of a propagator $p$ be referred to as $c_p$. A 
\emph{constraint model} is a combination of a domain and a set of propagators
$\langle d, P \rangle$. This is very similar to a CSP as defined
above, and a model can be transformed to a CSP using
$csp(\langle d, P\rangle)=\langle d, \{c_p |\forall p \in P\}\rangle$. The crucial
difference is that a constraint model can define constraints in
intension, instead of the extensional full set of solutions in a
CSP. Another view is that the CSP defines the semantics, and a model
defines how to compute solutions to a problem.

Solving a model is typically done by interleaving fix-point
computation of the propagators with search using heuristic
decomposition of the model (branching or labeling). We leave the
details of solving opaque for now, assuming a function $\solve$ for
CSPs where $solve(\langle d, P\rangle)$ returns all solutions that are
fix-points of all propagators.

In~\cite{SchulteTack:CP:2009}, Schulte and Tack introduced weak
monotonicity and showed that the above properties for
propagators\footnote{The \emph{local} property was not needed there,
  as their constraints and propagators are defined over all variables.} are
necessary and sufficient to get \emph{sound} and \emph{complete}
solving when combined with search; when solving a model all solutions
found are solutions for all the constraints, and all solutions that
satisfy the constraint are found. It is common to require monotonicity
from propagators ($\forall d_1,d_2\in\Dom.\, d_1\subseteq d_2\implies
p(d_1)\subseteq p(d_2)$), but this does not model actual propagators
well, since it excludes many types of random and heuristic
propagators. The gain from having monotonic propagators is that the
fix-point of all the propagators is unique, regardless of the order of
propagators run.

In practice, a single constraint may be implemented by a set of
propagators, such as $n^2$ not equals propagators for an
\cons{all\_different} constraint. We will leave this generalization out
of the formalization, but note that it is straightforward.

Propagators are often characterized on their propagation
\emph{strength}.  Given two propagators $p_1$ and $p_2$ for a
constraint $c$, $p_1$ is \emph{stronger} than $p_2$ iff for all
domains $d$, $p_1(d)\subseteq p_2(d)$, and for some domain $d'$,
$p_1(d')\subset p_2(d')$. A \emph{consistency level} defines a
specific strength of propagation. The canonical example is
\emph{domain consistency} (also called (generalized) arc consistency,
or complete propagation), where a propagator $p$ is domain consistent
iff $\forall d\in \Dom. p(d) = dom(asn(d) \cap c_p)$. That is, the
propagator removes all values for variables that have no supporting
assignment in the associated constraint. Domain consistency is
interesting since it is the strongest consistency possible, without
violating the requirements for a propagator.  There are other
consistency levels defined in the literature, for example \emph{value
consistency} (also called forward checking), and \emph{bound
consistency}.

\subsection{Constraint programming systems}
\label{sec:cp:systems}

Constraint programming systems are designed to enable the
specification and solving of constraint models.  Typical examples
include open source solvers such as Gecode~\cite{gecode},
Choco~\cite{choco}, and OR Tools~\cite{ortools} and commercial solvers
such as SICStus Prolog~\cite{sicstus} and CP
Optimizer~\cite{cpoptimizer}.

Constraint programming systems contain implementations for
\begin{description}
\item[Variables] Variables can be Booleans, integers, floats, sets,
  and so on. 
\item[Propagators] Propagators are the implementations of
  constraints. Systems typically provide many different propagators,
  for many different constraints.
\item[Branching] A branching is an implementation of a heuristic, that
  decides how to make guesses in a search tree. 
\item[Search] Search is used to find solutions to models comprised of
  variables and propagators combined with branchings. Search methods
  can be complete (DFS, Limited Discrepancy Search) or incomplete
  (Restart based search, LNS), and can be for solutions only or
  finding optimal values.
\end{description}

For implementing search, systems need to provide support for state
restoration~\cite{Reischuk09}. The two main types are trailing and
copying + recomputation. Trailing involves keeping a trail that
encodes undo-information, so that when backtracking in a search tree the
changes along a path can be undone. Copying and recomputation works by
keeping a list of \emph{redo} information, typically the branching
decisions taken, combined with regular check-pointing of the state
using copies.

\section{Half-checking propagators}
\label{sec:hf-propagators}

A \emph{half-checking} propagator is similar to a traditional
propagator, only with less restrictions. In particular, half-checking
propagators are allowed to actually \emph{remove solutions}.  A
half-checking propagator is a propagator that only requires that if a
solution is detected, then it is correct. Formally, a half-checking
propagator is a function from domains to domains, with the properties
that it is \emph{local} and \emph{contracting}, in addition to the
following property:

\begin{definition}[Half-checking]
  The propagator $p$ is \emph{half-checking} for $c$, if for all
  assignments $a\in\Asn$, if $p(\dom(\{a\}))=\dom(\{a\})$ then
  $a\in c$.
\end{definition}

Half-checking is a natural weakening of \emph{checking}, where instead
of requiring that all solutions to a constraint are precisely
identified and thus the only fix-points of the function, we only
require that fix-points of assignments must be solutions to the
constraint. Also importantly, a half-checking propagator is not
required to be \emph{weakly monotonic} either. Since weak monotonicity
is required for correctness,
a half-checking propagator may actually be \emph{incorrect}: it may
remove an assignment that it would recognize as a solution.

\begin{example}
  The \emph{fail} propagator $\lambda d.\bot$ is a half-checking
  propagator for all constraints $c\in\Con$. Since \emph{fail} has no
  fix-points for any assignment, it is trivially half-checking. It is
  naturally contracting, as well as local, since all empty/failed
  domains are equated. Note that the \emph{fail} propagator is the
  strongest propagator possible, since
  $\forall d\in \Dom.\, \bot\subseteq d$. (Note also that \emph{fail} is
  a rather useless propagator in practice, since it guarantees that no
  solution will be found.)
  \label{ex:fail}
\end{example}

\begin{theorem}
  All propagators are also half-checking propagators.
\end{theorem}
\begin{proof}
This follows directly since \emph{half-checking} is a weakening of
 \emph{checking}.  
\end{proof}

\begin{theorem}
  Solving a constraint model with half-checking propagators using
  $\solve$ is \emph{sound}.
\end{theorem}
\begin{proof}
  All returned solutions from $solve$ must be fix-points for all the
  propagators (by definition, whether traditional or
  half-checking). Since the only fix-points of both traditional and
  half-checking propagators are solutions to the associated
  constraint, the returned assignments are solutions to the model.
\end{proof}

\begin{theorem}
  Solving a constraint model with half-checking propagators using
  $solve$ is \emph{not complete}.
\end{theorem}
\begin{proof}
  Given is a model $\langle d, P\rangle$ with at least one
  solution. We can replace any propagator $p$ in $P$ with \emph{fail}
  from Example~\ref{ex:fail} as a half-checking propagator for the
  constraint $c_p$. With \emph{fail} in the set of propagators, no
  solutions are produced since there are no assignment fix-points for
  \emph{fail}.
\end{proof}

\section{Integrating half-checking propagators into a system}
\label{sec:integration}

After defining and describing the theoretical properties of
half-checking propagators, it is important to investigate how they can
be supported in constraint programming systems.  In most constraint
programming systems, propagators are just components that interact
with the current variables, and based on deductions may remove some
values from its variables domains. 

When implementing a propagator in a typical constraint programming system,
the properties \emph{contracting} and \emph{local}  are natural
consequences of the programming interface: propagators only have
access to their variables, and the only modifications that a
propagator can do are removal of values from domains.

As shown in~\cite{SchulteTack:CP:2009}, a constraint programming
system that uses re-computation may need to make adjustments for
weakly-monotonic propagators as opposed to monotonic propagators. The
reason is that running propagation twice may not give the exact same
result, since the fix-point is no longer unique. Typical examples of
this might be propagators that use randomized algorithms. The same situation
naturally applies for half-checking propagators, and thus if the
system is set up such that it can handle weakly-monotonic propagators,
it can also handle half-checking propagators.

In addition to supporting half-checking propagators, there are
additional practical concerns that need to be taken into account. When
applicable, we will describe how this is done for the Gecode system.

\subsection{Portfolio-based search}
\label{sec:portfolio}

Using half-checking propagators naturally leads to a incomplete
search. In many cases, this may be ok and a desired outcome, but
sometimes a user would like to know that all solutions have been
found, that no solution exists, or that the optimal solution has been
found. Using a cooperative portfolio solver combining an incomplete
search with a complete search solves this, such as in the Failure
Directed Search~\cite{failure-directed-search} used in the CP
Optimizer~\cite{cpoptimizer} system. Portfolios of
solvers, with some assets incomplete, for scheduling problems is
explored in~\cite{Fontaine16}.

It is important to indicate to the portfolio system used that the
asset with half-checking propagators is not a complete search
method. If it is not possible to inform the system that an asset is
incomplete, the resulting combined search may wrongly indicate that it
is complete. In Gecode, returning false from the function called to
set up the asset indicates that the asset is incomplete.

Given many half-checking propagators, there are three main ways in which
they can be used together in a portfolio system.
\begin{description}
\item[Combined] Half-checking propagators can naturally be combined
\item[Multiple assets] For each half-checking propagator, create an
  asset in the portfolio that runs the problem with it. This may
  require creating many assets.
\item[Round robin] To avoid too many assets, a single asset can be
  used with a round-robin schedule that upon re-start switches between
  the different half-checking propagators to use.
\end{description}
Which strategy to use will depend on the problem at hand, the
half-checking propagators, and the instances to solve. For any
particular problem, it will require experimentation combined with
experience in the behaviour of the half-checking propagators in question.

\subsection{No-good recording}
\label{sec:no-good}

A crucial aspect for modern re-starting search is to record
no-goods~\cite{Katsirelos05,Lee16}. A no-good is a constraint that
describes the search-tree that has been explored so far, and is added
upon re-start. In constraint programming, no-goods are typically based
on negating the conjunction of a set of branching decisions. When combined with traditional
constraint propagation for monotonic propagators, branching decisions
precisely describe the explored part of a search tree. For
weakly-monotonic propagators, the search-tree may not be precisely
described by the no-good, but it is still correct. 

In the presence of half-checking propagators, the parts of a
search-tree that have been visited may contain solutions that were
removed. Thus, a no-good from a
search using half-checking propagators \emph{is not globally valid}. It is
still useful in the search using that half-checking propagator, but if
it is used in an asset that claims to be complete, this will no
longer be true.

Consider again the \emph{fail} propagator from
Example~\ref{ex:fail}. Given a portfolio search with one asset a traditional
and complete search, and one asset using \emph{fail}. As soon as the
latter is run it will fail and be done. Recording the no-good and
posting it in the traditional asset will abort the search since the no-good
would rule out the whole search tree.

\subsection{Lazy clause generation}
\label{sec:lazy}

In lazy clause generation solvers~\cite{lcg}, a
propagator \emph{explains} its deductions using clauses. There is
nothing inherently problematic about combining half-checking
propagators and lazy clause generation. One interesting aspect, is
that a simple half-checking propagator that does some very mild extra
deductions may produce clauses that are later on used in the
no-good explanation clauses generated on failure, and may thus end up
being used in a wider context.

For some half-checking propagators, such as the removal of crossing
edges described in Section~\ref{sec:hf-examples:dominating},
generating good explanations is easy. For others, such as the
approximation based upper bound computation in
Section~\ref{sec:hf-examples:bounds}, useful explanations can be
generated if the approximation produces a witness solution. However,
for some half-checking propagators such as the heuristic based filtering in
Section~\ref{sec:hf-examples:solutions}, explanations may be quite
hard to produce.

\subsection{Testing of propagators}
\label{sec:testing}

Propagators are complicated pieces of code, and testing is naturally
needed to increase the confidence that a constraint programming
system produces the correct results.  Unfortunately, half-checking
propagators make the job of testing harder, since there are fewer
guarantees that we can rely on.

Testing in the Gecode system is based on a kind of test oracles using
a set-up that combines initial domains with a constraint checker. A
constraint checker is typically a much simpler piece of code to write
than the propagator under test. For all assignments in the initial
domains, the testing system then removes values towards that
assignment, running the propagator under test intermittently. If the
assignment is in the constraint/validated by the check, the propagator
should not remove the assignment, and otherwise the search should
eventually fail. The whole idea relies on weak monotonicity, which
half-checking propagators do not have. In addition, propagators may
opt-in for extended checking of bounds and domain consistency, neither
of which are useful to a half-checking propagator.

In~\cite{ozgur18} metamorphic testing is used to test constraint
propagators. The idea is to use an extensional constraint with a table
propagator as a
validation propagator. A test consists of running original propagator
and the validation propagator, and then comparing the resulting search
trees. Again, the fact that a propagator must be weakly monotonic and
checking are crucial properties here. 

A similar idea
is explored in
SolverCheck~\cite{solvercheck}: initial domains and a constraint
checker are used to generate a list of valid assignments. These
assignments are then used to build reference propagators, including
weakening them to build bounds-consistent propagators. Propagation of
the propagator under test is compared with the simple
reference propagator. Again the assumption is naturally that
propagators are correct, and will not remove solutions.

Since half-checking propagators are allowed to remove solutions, none
of the above testing strategies will work.  However, there are some
things that we could test for, namely the half-checking
property. Using the Gecode testing strategy, it is possible to adjust
it to only check that a solution accepted by the propagator was also
verified by the checker as being valid.

In Section~\ref{sec:hf-examples:bounds} half-checking propagators that
update bounds based on approximations are described. These may use
inferences that are always valid for optimal solutions. Thus, by only
considering optimal assignments in a Gecode-style testing set-up, the
propagator \emph{can} be tested for optimal assignments in the
traditional manner.

Naturally, many half-checking propagators may use standard
algorithms, and these can of course be tested using any normal kind of
testing framework.

\section{The \cons{cost-circuit} constraint and TSP}
\label{sec:tsp}

In the following sections, examples of general techniques and
strategies to use when implementing half-checking propagators are
given. For each one, an algorithm is proposed for the \cons{cost-circuit}
constraint. This section describes the
constraint and the Travelling Salesperson Problem that it is used for.

\subsection{Theory}
\label{sec:tsp-theory}

Let $G=\langle V,E \rangle$ be a graph consisting of a set of vertices
or nodes $V$ and a set of edges $E\subseteq V\times V$ indicating
which edges are connected. The \emph{degree} of a node is the number
of edges connected to it. The graph is \emph{complete} if
$E= V\times V$, i.e., all nodes are connected to all other nodes (the
degree of each node is $|V|-1$). The
graph may be \emph{directed} or \emph{undirected}. A \emph{path} of
length $k$ in a graph is a sequence of nodes
$\langle v_1, v_2, \ldots v_k \rangle$ where
$\forall_{i\in 1\ldots k-1}.\, \langle v_i,v_{i+1}\rangle\in E$. A path is
a \emph{circuit} when $\langle v_k,v_1\rangle\in E$. When all nodes
are unique it is called a \emph{simple path} and a \emph{cycle} or a
\emph{simple circuit}. When a simple path or a simple circuit covers
all the nodes ($k=|V|$), it is called Hamiltonian, and finding such
 are one of the classical NP-complete problems~\cite{Karp1972}. A
graph is \emph{connected} when there exists a path between all pairs
of nodes. A \emph{tree} is a graph that is connected and has no
cycles.  A \emph{weight function} $w$ is a function from edges to real
numbers ($w\in E\to \mathbb{R}$), and most often to non-negative real
numbers. It is \emph{symmetric} if
$\forall_{v_1,v_2\in V} w(\langle v_1,v_2\rangle) = w(\langle
v_2,v_1\rangle)$. A weight function \emph{respects the triangle
  inequality} when
$\forall_{v_1,v_2,v_3\in V} w(\langle v_1,v_3\rangle )\leq w(\langle
v_1,v_2\rangle ) + w(\langle v_2,v_3\rangle )$. Given a graph
$G=\langle V, E\rangle$ and a weight function $w$, a \emph{minimum spanning
tree} (MST) $M=\langle V, T\rangle$ is a tree with the same nodes as the
graph, with $T\subseteq E$, and with a minimum weight.

The \emph{Travelling Salesperson Problem} (TSP) is the problem of
given a graph $G=\langle V,E\rangle$ and a weight function $w$, find a
Hamiltonian circuit for the graph with minimum weight. This is the
natural weighted extension of the Hamiltonian path problem. It is
common to require that the graph for a TSP is complete; a missing edge
can be modelled as an arbitrary large weight, and using bounds on
weights to check feasibility. If the nodes of the graph have positions
and the weight is defined as the distance between the nodes, it is a
\emph{Euclidean} TSP. The TSPLIB~\cite{tsplib} is a collection of 110
challenging real-world TSP instances, with 77 of these using Euclidean
2D-distance.

\subsection{TSP in constraint programming}
\label{sec:tsp-cp}

The $\cons{circuit}(S)$ constraint models the Hamiltonian circuit problem
using an array of successor variables $S$, where $S_i = j$ indicates
that $j$ is the successor of $i$ in the circuit. The
$\cons{cost-circuit}(S,w,c)$ is the same, with the variable $c$
representing the total cost of the circuit according to the weight
function $w$.

The \cons{circuit} constraint is one of the classical global
constraints in constraint programming~\cite{Lauriere78,Beldiceanu94}. Since the base problem is
NP-complete, filtering algorithms are focused on effective but not
complete filtering. The base filtering is handled by the embedded implied
$\cons{alldifferent}(S)$, with additional removal of edges that would lead to
circuits smaller than $|S|$ (subtour elimination). In addition, many other structural filters
have been identified and
propagated (e.g.,~\cite{SchulteTack:CP:2009,Francis14}).  For
the weighted variant, there have been recent advances above the basic
filtering, for example in~\cite{Benchimol12}
and~\cite{Isoart19}.

The above propagation algorithms are all limited by the fact that no
correct value may be removed. State of the art TSP solvers such as
Concorde~\cite{concorde} can do more, since the goal is to find a
single optimal solution, not all possible solutions.

In constraint programming, the choice of the branching heuristic is
key. For TSP, several different heuristics have been
proposed~\cite{Fages16,Isoart19}, with no clear
winner. Here, we will focus on the \emph{Warnsdorff}
heuristic~\cite{wan1823} for the Knights tour problem (and more
generally, the Hamiltonian path problem). The heuristic is, when cast
in constraint programming terms, comprised of two
parts. The first is the variable ordering, assigning variables along a
path that is built up incrementally. The second is the value ordering,
preferring to go to nodes with the lowest out-degree. Adjusted for the
case of complete graphs with distances, the out-degree is less
important and using the minimum distance becomes more important.

\section{Technique: Dominating solutions}
\label{sec:hf-examples:dominating}

When solving a constraint programming problem, it is common to see
that one solution may \emph{dominate} another solution, either because
of symmetries or because of one solution having better
cost. Propagation for symmetries is common~\cite{matrixsymmetry}, as
is more global views for symmetry breaking~\cite{LDSB}. For
cost-dominating solutions, there is less opportunities for
incorporating the domination relation into propagators, since it is
typically quite specialized and will not behave as a traditional
propagator. This is a clear opportunity to apply half-checking
propagators.

\subsection{No Crossing Lines}
\label{sec:ncl}

In a pure Euclidean TSP over a complete graph with no side-constraints, a
property that always holds is that in a optimal solution there are no
crossing lines: given two crossing lines
$\langle s_1,e_1 \rangle$ and $\langle s_2,e_2 \rangle$, they can be
replaced with $\langle s_1,e_2 \rangle$ and $\langle s_2,e_1 \rangle$,
which will have the same or lower weight.  Thus, any solution that
contains crossing lines will be dominated by a solution in which the
crossing lines are un-crossed. For an edge $e$, let
$cl(e)\subset E$ be the set of lines that cross it.

Using this observation, we can design our first interesting
half-checking propagators, which we call $\cons{ncl}(S)$ for No
Crossing Lines. The key
observation is that given an assignment that includes an edge $e$ in
the solution, we known that in no \emph{optimal} solution where $e$ is used
(if any such exist), are any of the lines in $cl(e)$ used. Note that
there may be no optimal solution including the edge $e$.
Given an assignment $S_i = j$, for all edges
$\langle k,l \rangle\in cl(\langle S_i,S_j \rangle)$, propagate
$S_k \neq l$.

For a solution that uses Warnsdorff's rule for variable selection, it
is possible to choose a simpler filtering called
$\cons{ncl-warn}(S,f)$. The propagator follows the Warnsdorff path
from the starting node $f$ to the last known node in the
path, and removes any outgoing edges from that node that cross the fixed path.

Stronger reasoning using crossing lines is also possible. For a
variable $S_i$ with domain $d_{S_i}$, any edge $\langle
k,l \rangle\in\cap_{v\in d_{S_i}} cl(\langle S_i, v \rangle)$ can be
removed. We have not implemented this stronger propagation. 

\label{sec:ncl-impl}

\paragraph{Implementation.} Implementing \cons{ncl} requires a fast and efficient look-up of the
$cl$ sets. Since the graph is fixed, we pre-compute this information.
In a complete TSP with $n$ nodes, the number of edges is $n^2$, which
means that the number of crossing lines is $O\left(n^4\right)$, a very large
number for a modest number of cities. Thus, the propagator can only be
used for quite small instances.

For \cons{ncl-warn}, the crossing lines are computed on the fly
instead. Along the Warnsdorff path, $n$ assignments will be made, and
for each assignment $O(n)$ other edges need to be considered. Thus,
along a path a maximum of $O(n^2)$ pairs of edges are considered. This
is much less taxing than the full \cons{ncl} propagation.

To speed up the computation of the crossing lines, a spatial
index was used to make geometric look-ups. Our index is based on the
STR construction of R-trees~\cite{STR}. We adjusted it in two
ways. The first is to make binary sub-divisions recursively. The
second is to first sort objects based on width/height, and then on
position. This strategy is useful since many edges are very long and
cover most of the other edges. Using this ordering instead of the
normal STR ordering gave a small but significant speed-up.

\section{Technique: Heuristic bounds}
\label{sec:hf-examples:bounds}

For many hard problems in computer science, there are algorithms
defined that create good but not provably optimal solutions. Such
algorithms are often constructive, meaning that they produce a witness
solution showing how to achieve the bound.

Bounds are typically used in constraint programming propagators for
the worst case, i.e., finding the lowest and the highest weight
possible. The difference here is that we instead strive to give good
and tight upper-bounds based on a best-effort to find a solution to a
single constraint. Naturally, such bounds may be invalid in the
presence of other constraints in the model, but if they are valid,
they will help guide propagation.

\subsection{Christofides bounds propagation}
\label{sec:christofides}

The classical approximation algorithm for Euclidean TSP is
Christofides algorithm~\cite{christofides}.  The algorithm is defined
for a complete graph $G=\langle V, E \rangle$ with Euclidean weights
$w$, and the outline is the following. 

\begin{itemize}
\item Find a minimum spanning tree of $G$, $M$.
\item Let $O$ be the set of edges with odd degree in $M$.
\item Find a minimum weight complete matching in $G$ among the nodes in $O$,
  and add these edges to $M$.
\item Construct an Euler circuit in $M$ (a circuit that crosses each
  edge once). Guaranteed to exist since all nodes have even degree.
\item Following the Euler circuit, skip any node that has been used
  before with the corresponding edge in $E$.
\end{itemize}

The resulting circuit is at most 1.5 times the length of the optimal
circuit. Note that the algorithm requires that the graph is complete.
The Christofides algorithm is very popular as a reasonably simple
algorithm that gives a good bound. For example, it is implemented as a
stand-alone TSP solver in OR Tools~\cite{ortools}.

We propose the $\cons{cbp}(S,w,c)$ bounds propagator, that works as
follows. Let $G_S=\langle V, E_S \rangle$ be the current graph induced
by the $S$ variables, with $G$ the original graph. For simplicity, we
treat the graph as undirected. Our algorithm proceeds as follows
\begin{itemize}
\item Find a spanning tree of $G_S$, $M_S$, with the fixed edges in
  $S$ included.
\item Let $O$ be the set of edges with odd degree in $M_S$.
\item Find a \emph{maximal} matching in $G_S$ for the
  edges in $O$, and add to $M_S$.
\item For the nodes not matched in the previous step, find a matching
  using the edges in $G$ and add to $M_S$.
\item Construct an Euler circuit in $M_S$.
\item Following the Euler circuit, skip any node that has been used
  before with the corresponding edge in $E$, even if it is not in
  $E_S$.
\item Adjust the upper bound of $c$ to be at most the weight of the
  found circuit.
\end{itemize}

The above algorithm tries as far as possible to use only edges in the
graph $G_S$. If only such edges are used, then the upper bound
represents a solution to the sub-problem. Otherwise, the best
remaining tour may have a larger cost.

\label{sec:christofides-impl}

\paragraph{Implementation} The implementation of the \cons{cbp} propagators follows the
outline above.
The spanning tree is found using Kruskals
algorithm~\cite{kruskal}. First all fixed edges are added to the
tree. After this, the edges in the graph are traversed in increasing
order. For this, our graphs keep a list of all the
edges in increasing weight order. If $|E_s|> \frac{1}{4}|E|$, then
this list is used with a filter to check for validity, otherwise a new
list is constructed from the current domains. The constant
$\frac{1}{4}$ was determined through experimentation, and needs to be
adjusted for a specific implementation.
Finding the Euler walk is done using Hierholzers
algorithm~\cite{hierholzer}, with the stack-based formulation.

The largest difference is that instead of the complete
minimum weight matching, a simple greedy algorithm is used
instead. This is because implementing and running a maximal matching
algorithm such as Edmonds algorithm~\cite{edmonds-matching} is both
complicated and time-consuming. An approximate solution here may give
a higher bound, but never a wrong one.

\section{Technique: Heuristic solutions}
\label{sec:hf-examples:solutions}

This is the most general technique, where heuristic algorithms are
used to make inferences and deductions that may or may not be true.

\subsection{Heuristic 1-tree propagation}
\label{sec:onetree}

As discussed in~\cite{HeldKarp71,Benchimol12,Isoart19}, a 1-tree is a
very useful structure for analysing properties of graphs when
searching for weighted Hamiltonian circuits.  Formally, a 1-tree for a
graph $G=\langle V E\rangle$ and a node $n_1$ is a spanning tree for
the graph
$\langle V\setminus \{n_1\}, E\setminus \{\langle n, n' \rangle |
n=n_1\lor n'=n_1\}\rangle$ along with a set of two edges from $n_1$ to
the rest of the graph:
$\{\langle n,n_1 \rangle, \langle n_1,n' \rangle \}$. A minimum 1-tree
is a 1-tree with minimum weight. Note that all circuits are 1-trees
for all the nodes in the graph as the selected node.

Our \cons{one-tree} propagator starts by finding a node to use as the
dedicated node, after which a 1-tree is computed. Three rules are
used: Update the lower bound of the cost with the cost of the 1-tree;
If the 1-tree is a circuit, set this as the solution; For some
node with degree $>2$ in the spanning tree part, remove the longest of
the incident edges. The latter idea is inspires by Held and
Karps~\cite{HeldKarp71} techniques from MIP formulations of the TSP
problem, where the residual costs of the edges in such nodes are
manipulated iteratively.

\label{sec:onetree-impl}

\paragraph{Implementation.} To find a 1-tree, the implementation uses an algorithm based on
Kruskals algorithm  similar to the implementation of
the spanning tree algorithm in~\ref{sec:christofides-impl}. The main
difference is that the special node $n_1$ is given as an additional
argument, and the algorithm returns a spanning tree for $V\setminus
\{n_1\}$, and two edges incident to $n_1$. First all fixed edges are
added, either to the spanning tree or to the $n_1$ edges. While
processing edges to build up the spanning tree, if an edge is incident
to $n_1$ add it to that set unless it already contains 2 edges. When
the spanning tree is constructed, we may still not have 2 edges in the
$n_1$ set, and if so add the smallest.
Note that the algorithm is only executed after normal propagation for
the circuit constraint has been done. Thus, we can assume that there
are at most 2 fixed edges incident to $n_1$.

\section{Evaluation}
\label{sec:evaluation}

\newlength{\cw}
\setlength{\cw}{1.2cm}

\begin{sidewaystable}
  \centerfloat
  \begin{tabular}{l p{0.2cm}p{\cw}p{\cw}p{\cw} cp{\cw}p{\cw}p{\cw} cp{\cw}p{\cw}p{\cw} cp{\cw}p{\cw}p{\cw}}
    \toprule 
               & & 
\multicolumn{3}{c}{\cons{wncl}} & & \multicolumn{3}{c}{\cons{cbp}} & & \multicolumn{3}{c}{\cons{one-tree}} & & \multicolumn{3}{c}{All} 
\\
\cmidrule(lr){3-5}
\cmidrule(lr){7-9}
\cmidrule(lr){11-13}
\cmidrule(lr){15-17}
    Instance
    & & $S$ & min & max  
    & & $S$ & min & max  
    & & $S$ & min & max  
    & & $S$ & min & max 
    \\
    \midrule
berlin52 &  & 99.54\% & $=$ & $=$     &  & $=$ & $=$ & 15.83\% &  & 99.95\% & 88.02\% & $=$    &  & 99.49\% & 88.02\% & 15.83\% \\ 
st70     &  & 99.87\% & $=$ & $=$     &  & $=$ & $=$ & 13.04\% &  & 99.97\% & 79.89\% & $=$    &  & 99.85\% & 79.89\% & 13.04\% \\
eil51    &  & 99.61\% & $=$ & $=$     &  & $=$ & $=$ & 17.39\% &  & 99.95\% & 90.57\% & $=$    &  & 99.57\% & 90.57\% & 17.39\% \\ 
eil76    &  & $=$     & $=$ & $=$     &  & $=$ & $=$ & 14.82\% &  & 99.98\% & 93.84\% & $=$    &  & 99.98\% & 93.84\% & 14.82\% \\
eil101   &  & 99.65\% & $=$ & $=$     &  & $=$ & $=$ & 11.72\% &  & 99.99\% & 93.78\% & $=$    &  & 99.63\% & 93.78\% & 11.72\% \\
lin105   &  & 99.84\% & $=$ & $=$     &  & $=$ & $=$ & 7.30\%  &  & 99.99\% & 61.94\% & $=$    &  & 99.83\% & 61.94\% & 7.30\%  \\
lin318   &  & $=$     & $=$ & $=$     &  & $=$ & $=$ & 4.87\%  &  & $=$     & 66.32\% & $=$    &  & $=$     & 66.32\% & 4.87\%  \\
pr76     &  & 99.89\% & $=$ & $=$     &  & $=$ & $=$ & 10.74\% &  & $=$     & 76.25\% & $=$    &  & 99.89\% & 76.25\% & 10.74\% \\ 
pr107    &  & $=$     & $=$ & $=$     &  & $=$ & $=$ & 5.83\%  &  & $=$     & 63.71\% & $=$    &  & $=$     & 63.71\% & 5.83\%  \\
pr124    &  & 99.58\% & $=$ & $=$     &  & $=$ & $=$ & 5.80\%  &  & $=$     & 73.48\% & $=$    &  & 99.58\% & 73.48\% & 5.80\%  \\
pr136    &  & 99.99\% & $=$ & $=$     &  & $=$ & $=$ & 8.08\%  &  & $\bot$  & $\bot$  & $\bot$ &  & $\bot$  & $\bot$  & $\bot$  \\
pr144    &  & $=$     & $=$ & $=$     &  & $=$ & $=$ & 5.50\%  &  & $=$     & 42.42\% & $=$    &  & $=$     & 42.42\% & 5.50\%  \\
pr152    &  & 99.40\% & $=$ & 99.49\% &  & $=$ & $=$ & 4.71\%  &  & $=$     & 58.93\% & $=$    &  & 99.40\% & 58.93\% & 4.71\%  \\
    
    \bottomrule\\
  \end{tabular}
  \caption{Filtering strength for the propagators. Reported is the
    reduction when using the propagators \cons{wncl}, \cons{cbp},
    \cons{one-tree}, and all combined on the domains size of $S$
    and the min and max cost after assigning 10\%. $=$ means no
    reduction, $\bot$ means a failure.}
  \label{tab:filtering-improvements}
\end{sidewaystable}

Our implementation\footnote{available at \url{https://github.com/zayenz/half-checking-propagators}}
is done using the Gecode~\cite{gecode} constraint programming system,
version 6.2.0. The main
constraint in the model is \cons{cost-circuit}, along with an
\cons{inverse} constraint to get variables representing the
predecessors also. The main branching heuristic used is the
Warnsdorff heuristic for selecting the variable to branch on, and for
values selecting the value with min weight (slightly
randomized). Instances are read from TSPLIB files.

The search uses a portfolio with several assets. Each asset runs a
restart-based search with a Luby-based restart schedule with a fairly
low scale, and no-goods are collected. The set-up with randomized
value selection and rapid restarts is inspired by~\cite{Archibald19}.

When a half-checking propagator is requested, it is placed in the last
asset, which declares itself to be an incomplete asset. For this
asset, no-good recording is also turned off, by modifying the branching
heuristic. Unfortunately, it is not possible in Gecode to known from
which asset a no-good is produced. The possibility to record no-goods
anyway is also included.

Our experiments are run on a a Macbook Pro 15 with a 6-core 2.7 GHz
Intel Core i7 processor and 16 GiB memory. The experiments
are not for deciding the best way to solve a TSP using constraint
programming, it is instead to demonstrate that the techniques adds
filtering power.

Computing the crossing lines data-structure from Section~\ref{sec:ncl}
quickly starts to
get costly. At around 50 nodes, it takes 0.25-0.3 seconds and at
around 100 nodes it takes 0.8-1.1 seconds. However, for
\texttt{lin318} with 318 nodes, it takes more than 5 minutes to
compute, which is clearly too long to be useful. In the following, we
will skip the full version since it is clearly impractical.

Table~\ref{tab:filtering-improvements} reports the filtering improvements for our
proposed propagators.  Five variants are run simultaneously, assigning
10\% of the nodes in the path. The variants use the standard model,
along with variants with the propagators \cons{wncl},
\cons{cbp}, \cons{one-tree}, and all three combined. The reported
value is the reduction in the sum domain size of the successor
variables $S$, and the adjustment of the minimum and maximum costs
compared with the standard model. As
can be seen, our propagators have complementary and strong filtering.

Finding good solutions quickly is naturally desired. Unfortunately,
the improved filtering does not translate into better solving
directly. For some test-cases, our propagators give modestly better
results for solving under time-limits. However, we believe that a main
issue is that it is not possible yet to generate no-goods local to an
asset in Gecode. Further investigation is clearly needed, as is
testing other problems using the \cons{cost-circuit} constraint.

\section{Related work}
\label{sec:related}

The requirement of correctness for propagators have been a constant in
constraint programming since the field began. Still, there are a few
techniques and approaches that have touched on similar ideas.

The most similar technique to half-checking propagators is
probably \emph{streamlining constraints}~\cite{Gomes04}. The original
idea is to post additional constraints in a model in order to focus on
certain subsets of solutions that exhibit some kinds of
regularities. Typically, these regularities are found examining
solutions to small instances, and the added streamliners help find
these regularities in larger instances. The idea is similar to
half-checking propagators, in that in order to solve a problem we may
want to rule out potential solutions. In a certain sense, the \cons{ncl}
no-crossing lines propagator is a streamliner constraint, since we
focus on the solutions that have the no-crossing lines regularity. On
the other hand, the Christofides bounds propagation (\cons{cbp}) and
the 1-tree propagation (\cons{one-tree}) we propose are not easily
formulated as streamliners. An additional difference is that
half-checking propagators focus on adding new reasoning for existing
constraints, while streamliner constraints focus on adding new
reasoning for models.

The similar approaches of cost propagation~\cite{GroheWedelin08} and
belief propagation~\cite{Pesant19} use a domain store that
indicates a common cost or belief for each variable-value pair. Both
approaches use the gathered information to guide the search (a
non-backtracking search for~\cite{GroheWedelin08}).  As remarked by
Pesant in~\cite{Pesant19} a value that gets a belief very close to
0 (or perhaps even 0, due to rounding errors), is very unlikely to be
in any solution, and thus it might be beneficial to actually prune these
values. Such a pruning rule would be a half-checking propagator.

In \cite{Benchimol12}, TSP instances tested are pre-processed with
tight bounds based on standard state-of-the-art heuristics. While it
is not clearly stated, this pre-processing is of course not valid if
there are any other constraints in the instances than just a
\cons{cost-circuit}. This kind of bounds updates is similar to what we
propose in Section~\ref{sec:hf-examples:bounds}, although we use it
continuously during search.

In~\cite{SellmannHarvey02}, Sellmann and Harvey propose using heuristic
constraint propagation. While it may sound similar to half-checking
propagators and especially the techniques we present, the crucial
difference it that Sellmann and Harvey focus on incomplete, but still
\emph{correct} propagation. 

\section{Conclusions}
\label{sec:conclusions}

This paper has introduced \emph{half-checking propagators}, a new
variant of propagators that are not required to be correct. Lifting
this restriction opens up new possibilities for designing propagation
algorithms. The goal is to guide search towards good solutions. To
regain completeness, we paired models with half-checking propagators
in a portfolio with standard models.

A detailed description on how to integrate half-checking propagators
into modern constraint programming systems was given. To showcase the
idea, three techniques for designing half-checking propagators were
presented and made concrete with an application to the
\cons{cost-circuit} constraint. 

\paragraph{Future work.} The most important future work is of course
to make computational studies on how to best use half-checking
propagators. In order to make this as fair as possible, an improvement
to Gecode that would allow us to record no-goods locally in assets
with half-checking propagators is needed.

There are many examples of hard problems, where half-checking
propagators could be useful. We think that scheduling problems may be
an interesting future area of research for this. Also, studying
automatically generated streamliner
constraints~\cite{Wetter15,Spracklen19} could be an interesting source
of ideas for new half-checking propagators.

\section*{Acknowledgments}
Thanks to the anonymous reviewers of this paper, which helped improve
and clarify many points and recommended additional references.

\bibliographystyle{splncs04}
\bibliography{references}

\end{document}